\documentclass[11pt,a4paper]{article}

\usepackage{url}
\usepackage{algorithm}
\usepackage{algorithmic}%
\usepackage{mathtools}
\usepackage{amsmath,amssymb,amsthm}
\usepackage{bbm}
\usepackage{graphicx}

\usepackage[comma,sort&compress]{natbib}

\usepackage{thmtools}

\declaretheorem[name=Theorem]{Thm}
\declaretheorem[within=section,name=Lemma]{Lem}
\declaretheorem[sibling=Lem,name=Definition]{Def}
\declaretheorem[sibling=Lem,name=Notation]{Not}
\declaretheorem[sibling=Lem,name=Proposition]{Prop}

\newcommand{\Spec}{\operatorname{Spec}}

\newcommand{\Tr}{\operatorname{Tr}}

\newcommand{\Euc}{\operatorname{Euc}}
\newcommand{\Trans}{\operatorname{T}}
\newcommand{\Orth}{\operatorname{O}}
\newcommand{\Unit}{\operatorname{U}}
\newcommand{\Pol}{\operatorname{Pol}}

\newcommand{\spn}{\operatorname{span}}

\newcommand{\calF}{\mathcal{F}}

\newcommand{\calI}{\mathcal{I}}

\newcommand{\ZZ}{\ensuremath{\mathbb{Z}}}
\newcommand{\RR}{\ensuremath{\mathbb{R}}}

\newcommand{\CC}{\ensuremath{\mathbb{C}}}

\newcommand{\KK}{\ensuremath{\mathbb{K}}}

\DeclareMathOperator*{\Id}{I}

\makeatletter
\providecommand*{\diff}%
        {\@ifnextchar^{\DIfF}{\DIfF^{}}}
\def\DIfF^#1{%
        \mathop{\mathrm{\mathstrut d}}%
                \nolimits^{#1}\gobblespace
}
\def\gobblespace{%
        \futurelet\diffarg\opspace}
\def\opspace{%
        \let\DiffSpace\!%
        \ifx\diffarg(%
                \let\DiffSpace\relax
        \else
                \ifx\diffarg\[%
                        \let\DiffSpace\relax
                \else
                        \ifx\diffarg\{%
                                \let\DiffSpace\relax
                        \fi\fi\fi\DiffSpace}
\makeatother

\usepackage{hyperref}

\usepackage{titlesec}
\titleformat{\section}
	{\normalfont\Large\bfseries\filcenter}{\thesection.}{1 ex}{}
\titleformat{\subsection}
	{\normalfont\normalsize\bfseries}{\thesubsection.}{1 ex}{}
\titleformat{\subsubsection}[runin]
	{\normalfont\normalsize\bfseries\filcenter}{\thesubsubsection.}{1 ex}{}

\usepackage[charter]{mathdesign}
\usepackage[mathcal]{eucal}

\usepackage[margin=1.2 in]{geometry}

\usepackage{authblk}

\begin{document}
\title{Learning with Algebraic Invariances, and\\
the Invariant Kernel Trick}

\author[1]{
Franz J.~Kir\'{a}ly
\thanks{\url{f.kiraly@ucl.ac.uk}}
}

\author[2]{ Andreas Ziehe
\thanks{\url{andreas.ziehe@tu-berlin.de}}
}

\author[2]{ Klaus-Robert M\"uller 
\thanks{\url{klaus-robert.mueller@tu-berlin.de}}
}

\affil[1]{
Department of Statistical Science,
University College London,\newline
Gower Street,
London WC1E 6BT, United Kingdom
}

\affil[2]{
Machine Learning Group,
Technische Universit\"at Berlin,\newline
Marchstrasse 23,
10587 Berlin, Germany
}

\date{}

\maketitle
\begin{abstract}
\begin{normalsize}
When solving data analysis problems it is important to integrate prior
   knowledge and/or structural invariances. This paper contributes by a novel framework for
   incorporating algebraic invariance structure into kernels. In
   particular, we show that algebraic properties such as sign symmetries
   in data, phase independence, scaling etc. can be included easily by
   essentially performing the kernel trick twice. We demonstrate the
   usefulness of our theory in simulations on selected applications such as
   sign-invariant spectral clustering and underdetermined ICA.
\end{normalsize}
\end{abstract}


\section{Introduction}\label{sec:intro}
The construction of algorithms that encompass problem specific
invariances has been an important line of research in pattern
recognition and machine learning. Properly incorporating invariances
into a model allows to constrain the underlying function class and
thus to increase on one hand generalization and on the other hand
resistance against outliers and robustness against nonstationarities
that are caused by the respective invariance transformations. In
principle, machine learning models could become invariant by learning
from large corpora of data that contain the respective transformations
with respect to which invariance needs to be achieved, e.g. ambient
noise in speech recognition \citep{speechref} or robustness
against translation, rotation or thickness transformations in
handwritten digit recognition, (e.g.~\cite{SimVicLeCDen92,LeNet,virtualSVMs,Jeb03b}). Two general lines of
research have enjoyed high popularity, kernel methods \citep{Vapnik95,BoserVapnik92,KPCA1998}, where kernels can be {\em
engineered} to reflect complex invariances or prior knowledge, e.g. in
Bioinformatics \citep{Zien,Jiang12}, or
(deep) neural networks where large amounts of data help to {\em learn}
a representation \citep{LeNet,Bengio,Hinton}.

We will contribute to kernel methods and address the fundamental
question how prior knowledge on invariances can be {\em directly}
incorporated into the kernel formalism. More specifically, the kernel
trick extracts features, and it is important to
ask how these kernel features can be made naturally invariant under
data-specific symmetries, such as sign change, mirror symmetry, common
complex phase factor, rotation, and so on.

In particular, we propose an algebraic method that we will call {\em
the invariant kernel trick}. It allows (i) a modification of any
(positive semi-definite) kernel into a suitable invariant kernel by (ii) applying the
kernel trick twice. Namely, properties of two kernels are combined --
one for the invariance, and one for the features. (iii) Invariant
kernels come without an increase in the computational cost of kernel
evaluation, and finally (iv) the derived kernel is canonical,
and shares fundamental properties with the original. The kernel
invariant trick can be readily applied to any kernel-based method and
naturally and immediately implements the desired algebraic invariance
structure.

To underline the versatility of our novel invariance inducing
framework, we exemplarily show sign-invariant clustering simulations
on toy data obtained by modifying the USPS data set, and sign- and scale invariant signal separation on the real world ``flutes'' data set. These experiments are primarily intended to illustrate ease of use, usefulness and the broad applicability of the
kernel invariant trick. While we provide the theoretical details of
many potential invariance inducing transformations, it is clearly
unfeasible to show simulations for all of them in combination with all
possible kernel algorithms, therefore we have focused on clustering, sign invariance, and sign-and-scale invariance.

In the next section we will lay out our theory on invariances, then
discuss the spectral clustering algorithm variant used in section 3
and finally provide experimental results and discussion.

\section{Learning with Algebraic Invariances}\label{sec:theorems}
\subsection{An invariant kernel example: clustering on the sphere}
 \label{sec:theorems.invexample}
 We first illustrate the main idea in a simple example. Suppose one is given data points in $\RR^n$, coming from two different classes. Suppose these classes form clusters up to sign - that is, for each data point $x\in \RR^n$ is considered equivalently to $-x$, e.g., if one is interested only in the span of the normal vector $x$. After flipping the sign for some of the points, they fall exactly in one of the clusters (see Figure \ref{fig:XOR}).

 \begin{figure}[ht]
 \begin{center}
    \includegraphics[width=0.75\linewidth]{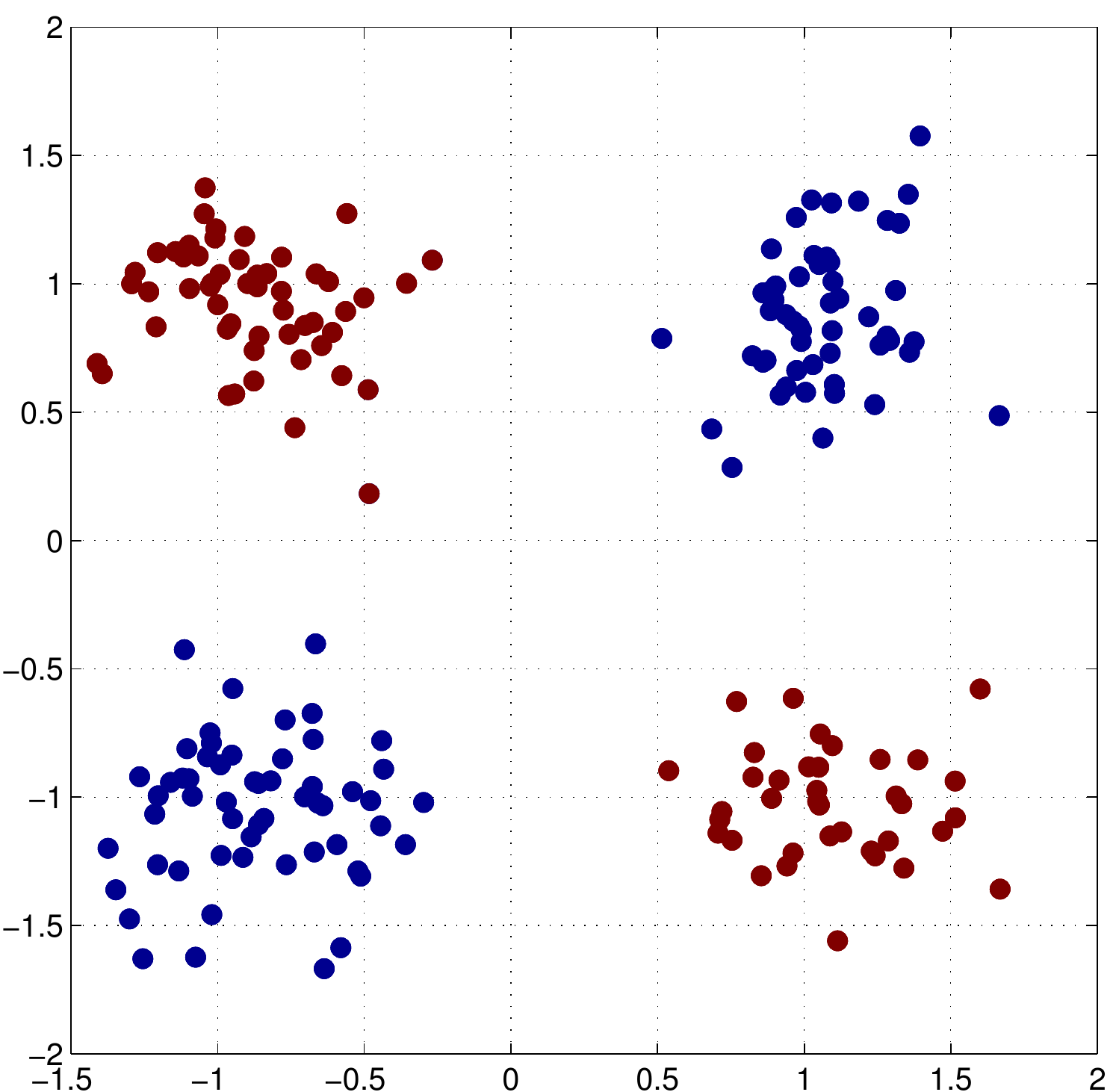}
 \end{center}
 \caption{The XOR dataset as a sign-invariant clustering problem.}\label{fig:XOR}
 \end{figure}

The first step is as follows: observe that for any $x\in\RR^n$, knowledge of the matrix $X=xx^\top\in\RR^{n\times n}$ is equivalent to knowledge of $x$ up to sign. Namely, $x$ can be obtained from $X$ by an eigenvector decomposition of $X$. One obtains $x=\pm v\cdot \|X\|$, where $v$ is the eigenvector of $X$ with biggest eigenvalue, and $\|X\|$ the Frobenius norm of $X$. Thus the matrix $X$ can be interpreted as a collection of invariant features.

On could now replace every data point $x$ by $xx^\top$, but through this maneuver the dimension of the data is roughly squared, thus the computational cost of learning in this representation increases considerably when compared to learning in the original ambient space. One may try to alleviate by computing random projections, e.g., clustering the vector of scalars $\Tr(x_ix_i^* A_j)$ for $i=1\dots M$ and $j=1\dots m$ for some $M< n^2$ and random/generic matrices $A_j\in\RR^{n\times n}$. But again, this is a heuristic step and can lead to inaccuracies and random fluctuations (even though it may work in some scenarios, and compressed-sensing-like guarantees can be derived for certain asymptotic settings).

We propose a different strategy which does not incur these problems. It combines the above invariant representation with the kernel trick and a kernel based learning method, for instance kernel $k$-means clustering. The typical kernel for clustering purposes is the Gaussian kernel
$$k(X,Y) = \exp\left(-\frac{\|X-Y\|^2}{2\sigma^2}\right).$$
The crucial second step is taking the kernel not with respect to the original data points $x$, but on the invariant features $X=xx^\top, Y=yy^\top$, obtaining a new kernel for $x$ and $y$ of the form
$$k_\pm(x,y):=k(X,Y)= \exp\left(-\frac{\|xx^\top-yy^\top\|^2}{2\sigma^2}\right),$$
where the norm is again the Frobenius norm on matrices. Applying the binomial expansion, one obtains
\begin{align}
 k_\pm(x,y) &=\exp\left(-\frac{\|x\|^4 + \|y\|^4-2\langle x,y\rangle^2}{2\sigma^2}\right).\label{eq:pgk}
 \end{align}
This is a sign-invariant variant of the Gaussian kernel which is indeed a (positive definite) kernel and none more expensive to evaluate than the original Gaussian kernel. It is further similar to the original by having an adaptable kernel width.

We will show and argue that these two steps are an instance of a very general and mathematically natural strategy for arbitrary kernels, which allows to construct modified invariant kernels for a large class of invariances. We will use the above example as a running illustration of the general strategy.

\subsection{Invariant kernels}
\label{sec:theorems.invariance}
We start by introducing an abstract and general {\bf setting} for invariances, which is as follows: we start with data points
$x_1,\dots, x_N\in \KK^n$, with $\KK = \RR$ or $\KK = \CC$, and an invariant group action (finite or continuous)
$$G\times \KK^n\rightarrow \KK^n,\quad (g,x)\mapsto g.x.$$
That is, an data point $x\in W$ is considered invariance-equivalent to the elements of the so-called orbit
$$G.x = \{g.x\;:\; g\in G\}.$$
We will assume (for technical reasons: ensuring existence of a well-defined quotient) that the group action is algebraic, that is, all maps $x\mapsto g.x$ are algebraic maps.
In the example presented in section~\ref{sec:theorems.invexample}, the group is the ``sign group'' $G=(\{+1,-1\},\cdot)$, with group operation $\cdot$ given by $+1=+1\cdot+1=-1\cdot-1$ and $-1 = +1\cdot-1 = -1 \cdot +1$, and group action given by $+1.x = x$ and $-1.x = -x$ (this group $G$ is, as a group, isomorphic to the group $\ZZ/2$).

Our main goal is to start from an arbitrary class of kernels $k:\KK^m\times \KK^m\rightarrow \RR$ (where the kernel is defined for any integer $m$ not necessary equal to $n$, such as for Gaussian or polynomial kernels) and make them invariant under the $G$-action. Mathematically, there are two meaningful but different ways for $k$ to be invariant:

\begin{Def}
Let $W\subseteq \KK^n$. Consider a group $G$ acting on $W$, and a function $k:W\times W\rightarrow \RR$. The function $k$ is called:
\begin{description}
\item[(i)] \emph{ $G$-invariant }, if $k(x,y)=k(g.x,h.y)$ for all $g,h\in G$ and all $x,y\in W$.
\item[(ii)] \emph{diagonally $G$-invariant }, if $k(x,y)=k(g.x,g.y)$ for all $g\in G$ and all $x,y\in W$.
\end{description}
If $k$ is a kernel, we similarly call it a $G$-invariant or diagonally $G$-invariant kernel.
\end{Def}

Note that $G$-invariant kernels are also diagonally $G$-invariant, but the converse is in general not true.

\subsection{Invariant kernels: characterization, existence and uniqueness}
Before constructing invariant kernels, we will characterize them abstractly. Similarly to the Moore-Aronszajn theorem, which asserts the existence of a unique feature space, a classical theorem of invariant theory asserts existence of a unique invariant space. We start by stating (technically convenient variants of) both theorems:

\begin{Thm}[Moore-Aronszajn]\label{Thm:Moore}
Let $W\subseteq \KK^n$ compact, let $k: W\times W\rightarrow \RR$ be a symmetric, positive definite kernel. Then there exist a $\KK$-Hilbert space $\calF$ and a continuous map $\phi: W\rightarrow \calF$, unique up to isomorphism, such that
$k(x,y) = \langle x,y\rangle_\calF\quad\mbox{for all}\; x,y\in W.$
\end{Thm}
The Hilbert space $\calF$ is called \emph{reproducing kernel Hilbert space} (RKHS) or \emph{feature space} associated to $k$. The map $\phi$ is called the~\emph{feature map}.

\begin{Thm}[Universal property of group quotient]\label{Thm:quot}
Let $W\subseteq \KK^n$ compact, and $G$ a group acting algebraically on $W$.
Then, there exist a $\KK$-Hilbert space of invariants $\calI$ and an algebraic map $q: W\rightarrow \calI$, unique up to isomorphism, such that any $G$-invariant (piecewise) continuous map $\phi: W\rightarrow \calF$ to a Hilbert space $\calF$ admits a unique factorization $\phi = \phi' \circ q$ with $\phi':´W/G\rightarrow \calF$ continuous, where $W/G = q(W)$. If $\phi$ is algebraic, then so is $\phi'$.
\end{Thm}
Here, $W/G$ is called the \emph{quotient space} of $W$ w.r.t.~$G$, and is also called orbifold if $G$ is finite and $W$ is a differentiable manifold. The map $q$ is called the~\emph{quotient map} or~\emph{invariant map}. It encodes the invariants of $W$ with respect to the group action $G$, and Theorem~\ref{Thm:quot} asserts that it does so in a unique and canonical way.

If $G$ is finite, it can be shown that $\calI$ is as well. If $G$ is infinite, $\calI$ needs not to be of finite dimension, but it can be shown to be countable always, and finite for all $G$ in this manuscript. See the appendix for more details. Both theorems also admit more technical variants where $W$ can be more general (e.g.~Moore-Aronszajn for non-compact $W$, group quotient for non-compact $W$), which we omit for sake of exposition and readability.

The two theorems can immediately be combined into a corollary which characterizes invariant kernels:
\begin{Thm}\label{Thm:invkchar}
Let $W\subseteq \KK^n$ compact. Consider a group $G$ acting on $W$, and a let $k:W\times W\rightarrow \RR$ be a (piecewise) continuous symmetric positive definite kernel.
\begin{description}
\item[(i)] If $k$ is $G$-invariant, then there is a symmetric positive definite kernel $k':W/G\times W/G \rightarrow \RR$, such that $k(x,y) = k'(q(x),q(y))$, where $q:W\rightarrow W/G$ is the canonical quotient map. $k'$ is unique (up to isomorphism of $W/G$), the feature spaces associated to $k$ and $k'$ coincide, and for the respective feature maps $\phi,\phi'$, it holds that $\phi = \phi'\circ q$, thus $k(x,y) = \langle \phi'(q(x)),\phi'(q(y))\rangle_\calF$.
\item[(ii)] If $k$ is diagonally $G$-invariant, then there is a function $f:(W\times W)/G\rightarrow \RR$, such that $k(x,y) = f(Q(x,y))$, where $Q:W\times W\rightarrow (W\times W)/G$ is the canonical (diagonal) quotient map.
\end{description}
\end{Thm}

It should be specifically noted that Theorem~\ref{Thm:invkchar} asserts a unique (up to isomorphism) decomposition of the kernel in an invariant part independent of $k$, given by $q$ resp.~$Q$, and a ``structural'' part, given by $k'$ resp.~$f$.

Some previous results on diagonally $G$-invariant kernels, including some statements on the corresponding feature spaces, can be already found in section~2.3 of~\cite{Haasdonk2007} and chapters~4.4 of~\cite{Kondor08}. We would like to note that they do not explicitly describe - therefore do not allow to explicitly construct - the invariant space in terms of the quotient map and its canonical invariants.

Intuitively, a $G$-invariant kernel is always a kernel \emph{on} the invariant features $q(x),q(y)$, while a diagonally $G$-invariant kernel is always a function in certain bivariate invariant functions $Q(x,y)$ (where in general $x$ and $y$ can not be separated).

Conversely, if one wants to construct kernels with certain invariance properties, the factorization Theorem~\ref{Thm:invkchar}~(ii) for diagonal $G$-invariances immediately implies that any positive definite function of the diagonal invariants $Q(x,y)$ will be diagonally invariant. Theorem~\ref{Thm:invkchar}~(i) for $G$-invariances implies that any class of kernels can be made into a $G$-invariant version by applying it to the features $q(x),q(y)$ - by the theorem, this procedure is canonical (up to isomorphism of the quotient $W/G$). The example in section~\ref{sec:theorems.invexample} constructs a $\ZZ/2$-invariant variant of the Gaussian kernel by observing that the quotient/invariant map is the map $q:x\mapsto xx^\top$.

\begin{Not}
In general, for a kernel $k:W\times W\rightarrow \RR$ and an invariant group action $G$, we will denote the $G$-invariant kernel canonically obtained by the procedure outlined above by $k_G$. If $k$ has a name, e.g.~Gaussian kernel, we call $k_G$ the $G$-invariant kernel of that name, e.g.~the $G$-invariant Gaussian kernel (keeping in mind that uniqueness holds only up to isomorphism of $W/G$).
\end{Not}

\subsection{The invariant kernel trick}
While the existence of $q$ or $Q$ is guaranteed by Theorem~\ref{Thm:invkchar}, the maps and associated invariant features can in principle be difficult or intractable to compute in practice - just as most of the common kernel features are difficult to obtain explicitly. The set of ideas outlined above is therefore only practically appealing when $q$ or $Q$ (or relevant parts thereof) are obtainable in (low-order) polynomial time, or in the following situation:

\begin{description}
\item[(a)] the original kernel $k:W\times W\rightarrow \RR$ is a function of Euclidean scalar products $k(x,y)=f(\langle x,x\rangle,\langle x,y\rangle, \langle y,y\rangle)$.
\item[(b)] the quotient/invariant map is the feature map of another efficiently computable kernel $\iota(x,y)$, that is, $\iota(x,y)=\langle q(x),q(y)\rangle$ where $q:W\rightarrow W/G$ is the quotient map.
\item[$\Rightarrow$] in this case, the $G$-invariant kernel may be obtained as $k_G(x,y) = f(\iota(x,x),\iota(x,y),\iota(y,y)).$
\end{description}
We call this the \emph{``invariant kernel trick''}, as the double application of the kernel trick allows us to avoid an explicit and potentially tedious computation of $q(x)$ and $q(y)$. In case of existence, we call $\iota$ the \emph{invariant kernel}. Note that in this case, $\calI = \mbox{span}\;W/G$ is the feature space associated to $\iota$.

Conditions (a) and (b) are fulfilled in section~\ref{sec:theorems.invexample}, but they may seem very special and a weak generalization of that example. This is, however, not true, as we argue in the following.

{\bf Condition (a)} is true for the majority of the more common kernels on $\KK^n$ in practical use, such as: Gaussian kernel, Laplace kernel and all other RBF kernels, homogenous and inhomogenous polynomial kernels, sigmoid kernel, ANOVA kernel. The only exceptions are kernels where a maximum/minimum is taken, such as spline and histogram kernels, and the more combinatorial kernels, e.g.~on strings and graphs, which in many cases do not take arguments in $\KK^n$. Moreover, the following result, which is an application of Theorem~\ref{Thm:invkchar}~(ii), implies that (a) is fulfilled for any kernel invariant under orthogonal/unitary, or Euclidean isometries; absence of such an invariance would imply an unparsimonious imbalancedness on the input representation.

\begin{Lem}\label{Lem:diagonalorth}
Let $k:W\times W\rightarrow \RR, W\subseteq \KK^n$ be a (piecewise) continuous function which is diagonally invariant under orthogonal/unitary transform, that is,~w.r.t.~the canonical group action of the orthogonal/unitary matrices $O(n)$ resp.~$U(n)$ on $\KK^n$.\\
Then, there is a function $f:\KK^3\rightarrow \RR$ such that $k(x,y)=f(\langle x,x\rangle,\langle x,y\rangle, \langle y,y\rangle)$.
\end{Lem}

\begin{Lem}\label{Lem:diagonalEuc}
Let $k:W\times W\rightarrow \RR, W\subseteq \KK^n$ be a (piecewise) continuous function which is diagonally invariant under the Euclidean isometries $\mbox{Euc}(n)$ on $\KK^n$.\\
Then, there is a function $f:\RR\rightarrow \RR$ such that $k(x,y)=f(\|x-y\|)$.
\end{Lem}

The converses of these Lemmas are easily seen to hold. Proofs of Lemmas~\ref{Lem:diagonalorth} and~\ref{Lem:diagonalEuc} follow from Lemmas~\ref{Lem:diagonalrot} and~\ref{Lem:diagonalrottrans} in the appendix, together with Theorem~\ref{Thm:invkchar}~(ii).

Lemma~\ref{Lem:diagonalEuc} states that isometry-invariant kernels can always be expressed as a radial basis function kernels $k(x,y) = f(\|x-y\|)$. Note that the isometry-invariant kernels are also unitary invariant, thus both fulfill condition (a) by Lemma~\ref{Lem:diagonalorth}; more explicitly, note that $\|x-y\| = \sqrt{\langle x,x\rangle -\langle x,y\rangle - \langle y,x\rangle + \langle y,y\rangle},$ thus RBF kernels fulfill condition (a).

{\bf Condition (b)}, on the other hand, is fulfilled for a variety of simple invariances, which we non-exhaustively list in the following section.

\subsection{Invariant maps and invariant kernels}
\label{sec:theorems.list}
We proceed by deriving invariant kernels for a list of common invariances. For exposition, we will explicitly write down the invariant versions of the Euclidean scalar product $k^E(x,y) = \langle x,y\rangle$, the inhomogenous polynomial kernel $k^\pi(x,y) = \left(\langle x,y\rangle+1\right)^d$, and the Gaussian kernel $k^\gamma(x,y) = \exp\left(-\frac{\|x-y\|^2}{2\sigma^2}\right)$.

\subsubsection{Finite rotation invariance}
This is a slight generalization of the running example from section~\ref{sec:theorems.invexample}: the group is $G=\ZZ/m$ acting on $\KK^n$, with the action being $\ell.x = \zeta^\ell_m x$, where $\zeta_m$ is a complex $m$-th root of unity. As $\zeta_2 = -1$, the case $m=2$ is sign invariance. For general $m$, the action rotates each coordinate of $x$ in the complex plane by an angle of $2\pi/m$. One can show: the invariant map is $q: x\mapsto x^{\otimes m}$, where $x^{\otimes m}$ is the $m$-th outer product tensor of $x$. The invariant space is the vector space of symmetric tensors of degree $m$.

An elementary computation shows that there is an invariant kernel $\iota(x,y) = \left\langle x^{\otimes m},y^{\otimes m}\right\rangle =\langle x,y\rangle^m$, which coincides with $k^E_G$ and the homogenous polynomial kernel of degree $m$. Note that the invariant map $q$ needs not to be evaluated in computation of the invariant kernel $\iota$. Using the invariant kernel trick, one further obtains $k^\pi_G(x,y) = \left(\langle x,y\rangle^m+1\right)^d$ and $k^\gamma_G(x,y) = \exp\left(-\frac{\|x\|^{2m} + \|y\|^{2m}-\langle x,y\rangle^m-\langle y,x\rangle^m}{2\sigma^2}\right)$.

\subsubsection{Phase invariance}
Phase invariance means that the data are in $\CC^n$, with $v\equiv \exp \left(2\pi i\cdot \varphi\right) v$ for any $\varphi\in \RR$ (and no further equivalences). The group is $G=\RR/\ZZ$, acting on $\CC^n$ as $\varphi.v = \exp \left(2 \pi i\cdot \varphi \right)\cdot v$. One can show that the quotient map is $w: v\mapsto vv^*,$ where $v^*$ denotes Hermite transpose of $v$. This leads to the invariant kernel $\iota(x,y)=\langle x,y\rangle\cdot \langle y,x\rangle$. The phase invariant kernels are in complete analogy to the sign invariant kernels.

\subsubsection{Scale invariance}
For data in $\RR^n/\{0\}$, scale invariance or projective invariance means that two data points are considered the same if they differ by a common scale factor, i.e., $v\equiv \alpha v$ for any $\alpha\in \RR^+$. The group is (additive) $G=\RR$, with the (multiplicative) action $\alpha.v = \exp (\alpha) \cdot v$. One shows that the quotient map is $q: x\mapsto \frac{x}{\|x\|}$, obtaining a scalar product $\iota(x,y) = \frac{\langle x,y\rangle}{\|x\|\|y\|}$. Interestingly, this is indeed a positive definite kernel, related to the so-called correlation kernel of~\cite{Jiang12}. The invariant kernel trick yields $k^E_G = \iota,$ $k^\pi_G(x,y) = \left(\|x\|^{-1}\|y\|^{-1} \langle x,y\rangle+1\right)^d$ and $k^\gamma_G(x,y) = \exp\left(\frac{\langle x,y\rangle}{\sigma^2\|x\|\|y\|}-\sigma^{-2}\right)$.

\subsubsection{Scale and sign/phase invariance or projective invariance}
The scale invariant feature map still leaves a sign/phase ambiguity. Having sign and scale invariance means $v\equiv \alpha v$ for any $\alpha\in\KK\setminus\{0\}$, the group is $G=\KK/\{0\}$ with action $\alpha.v = \alpha \cdot v$. One way to obtain the invariant kernels is via direct computation of the invariants - or one can apply the invariant kernel trick twice, applying either scale invariance to the sign/phase invariant kernel or sign/phase invariance to the scale invariant kernels. Either way, one finds the scale and sign/phase invariant kernel $\iota(x,y) = \frac{\langle x,y\rangle\cdot \langle y,x\rangle}{\langle x,x\rangle \cdot \langle y,y\rangle}$. Applying the invariant kernel trick with this invariant kernel (thus effectively the original kernel trick thrice) yields $k^E_G = \iota,$ $k^\pi_G(x,y) = \left(\|x\|^{-2}\|y\|^{-2} |\langle x,y\rangle| ^2+1\right)^d$ and $k^\gamma_G(x,y) = \exp\left(\frac{\langle x,y\rangle \langle x,y\rangle}{\sigma^2\|x\|^2\|y\|^2}-\sigma^{-2}\right)$.

\subsubsection{Multiple invariances}
In general, as demonstrated in the example of scale plus sign/phase invariance one can combine multiple invariances by repeating the invariant kernel trick with different invariances. If actions are compatible, the sequence of groups $G_1, G_2, \dots, G_k$ does not matter, as the final quotient will be the quotient of the group generated by the $G_i$. Care should be taken only for invariances where the group action has a different presentation, as the equivalence relation can trivialize, thus mapping all points to a single one in the quotient.

\subsubsection{Matrix invariances}
Another class of invariances which is practically relevant are invariances of matrix groups acting on $\KK^{mr} \cong \KK^{m\times r}$, for example the group of $(r\times r)$ permutation matrices or unitary matrices acting as $U.X = XU$. Permutation invariances give rise to bag-of-features representations, while unitary action is related to the empirical moments of the data rows. Row-wise translation is $\KK^m$ acting on $\KK^{m\times r}$ through $\mu.X = X+ \mathbbm{1}\mu^\top$ with $\mathbbm{1}$ being an $m$-vector of ones, leading to centered invariant kernels.

As quotient spaces for matrix invariances are more difficult to characterize than the ones presented, and so are the invariant kernels, we postpone the presentation to a longer version of this report.

\subsection{Related concepts and literature}

Both $G$-invariant and diagonally $G$-invariant kernels have been to our knowledge first defined and studied by~\cite{Haasdonk2007}, where they are called \emph{totally invariant} and \emph{simultaneously invariant} (as naming is not consistent throughout literature, we have decided to adopt a notation more closely inspired by invariant theory). Proposition~7 in~\citep{Haasdonk2007} asserts existence of the invariant feature space for $G$-invariant kernels. The mathematical concepts described in the paper differ from our approach in our opinion by being less general and less explicit: (a) The two main contributions in~\citep{Haasdonk2007} are two specific kernels which are $G$-invariant the translation integration (TI) kernel, and the invariant distance substitution (IDS) kernel. The TI kernel requires a Lie group endowed with Haar measure, and is is given in terms of an invariant Haar integral; the IDS kernel on first sight looks similar to our invariant kernel trick, but instead of applying the kernel trick twice, IDS kernels require a metricized group and substitutes hard-to-compute infimal distances into an RBF kernel; the IDS-type generalization of inner product further requires the choice of an (possibly arbitrary) origin. (b) In~\citep{Haasdonk2007}, the existence of an $G$-invariant feature space is proven, but no further explicit statements on its structure are made. In particular, a result such as Theorem~\ref{Thm:invkchar} which explicitly describes the form of kernels in terms of canonical invariants, given by the quotient map, is not available. On the other hand, our results not only show that the invariant kernel trick is universal, but also that any TI and IDS kernels as constructed in~\citep{Haasdonk2007} must obey the specific form of Theorem~\ref{Thm:invkchar}, implying that they could also be obtained by the invariant kernel trick, potentially avoiding computation of integrals and infimal distances on Lie groups.

One further related reference is~\citep{Kondor08}, chapter~4.4 and following, where \emph{diagonally} $G$-invariant kernels are studied, which behave quite different from $G$-invariant kernels (see above). Definition~4.4.1 refers to a definition of diagonal $G$-invariance in the reference [Krein, 1950], which we were not able to obtain at time of submission. Theorem~4.4.3 is a direct application of the Reynolds formula to diagonal invariance - it implicitly leads to the definition of a (non-diagonal) $G$-invariant kernel $k^G$. As we understand, $k^G$ is defined only for finite $G$ and furthermore needs a diagonally $G$-invariant kernel (which can be more difficult to obtain) to construct a $G$-invariant kernel. Moreover, $k^G$ is not used in the remainder of the manuscript, as the intention of 4.4.3~seems to be the relation of what in our terminology would be the diagonally $G$-invariant kernel and the feature space of the (non-diagonally) $G$-invariant kernel.

There is some further literature on diagonally $G$-invariant kernels, see e.g.~\citep{walder2007learning}. As $G$-invariant kernels are diagonally $G$-invariant, the respective results also hold for $G$-invariant kernels.

There is also an approach outlined by~\cite{chapelle2001incorporating} which aims at incorporating invariances with differential methods. However, the invariances discussed there are motivated by vertical/horizontal translation of a pixel image encoded as an element of $\RR^{m\times n}$, which is \emph{not} an invariance in our sense, since translation by pixels does not give rise to a group action on $\RR^{m\times n}$. Therefore our findings on the constructions of $G$-invariant kernels do not contradict the introductory statement of~\cite{chapelle2001incorporating} that globally invariant kernels do not exist, simply because the concept of invariance in~\citep{chapelle2001incorporating} is different and specific to the pixels image application, as~\cite{walder2007learning} already correctly note in their introduction.

\section{Algorithms}\label{sec:algorithms}

It goes without saying that the proposed invariant kernels can be used
in {\em any} kernel-based learning algorithm, see \citep{KPCA1998,Scholkopf02}. In order to
study the performance of our method we focus on a particular
unsupervised learning technique: kernel-based spectral clustering.

\subsection*{Kernel-based spectral clustering}
Spectral clustering methods have been pioneered by \cite{MeilaS00,Ng02}
   and have been studied and extended by
 e.g.~\cite{Dhillon04,Perona2005,Luxburg07,Sugiyama2014}.

In recent work it has been shown that kernel entropy component
analysis (kernel ECA) can be used to cluster data \citep{KECA2010}.
As kernel PCA, kernel ECA is based on an eigen decomposition of the
kernel matrix $K$ but identifies those kernel principal axes that
contribute most to the R\'enyi entropy estimate of the data.  It has
been demonstrated that selected kernel ECA components provide
information about the clustering structure of the data by analysing
the eigenvalues and the eigenvectors of the kernel matrix \citep{KECA2010}.

In order to demonstrate that any existing kernel algorithm can in principle profit from the concepts introduced above, we will in the following use the kernel-based spectral clustering algorithm described in~\citep{KECA2010} with Gaussian kernels invariant w.r.t.~different symmetries.

\section{Experiments}\label{sec:experiments}

In section \ref{sec:theorems.invexample} we introduced a novel class
of Mercer kernels that are invariant to sign flips of the input data.

As an example, in this section we consider the problem of spectral
clustering under sign invariances.  For our experiments we used the
MATLAB implementation of kernel ECA by Robert
Jenssen \footnote{http://ansatte.uit.no/robert.jenssen/software.html}
and our MATLAB implementation of the sign-invariant Gaussian kernel in
Eq. (\ref{eq:pgk}).

\subsection{Spectral clustering of handwritten digits}

In order to illustrate clustering under sign invariance we choose
98 images of handwritten zeros and ones from the USPS dataset \citep{USPSdata} where we randomly switched the signs of the  $16\times16$ pixel images. The input data consists of $98$ vectors in $256$ dimensions. We apply spectral clustering using kernel ECA with (a) a standard Gaussian kernel and (b) the sign-invariant Gaussian Kernel.
The results in figure \ref{fig:digits} show
that in the invariant case  a successful grouping into groups of zeros and ones, respectively, occurs, i.e.~ignoring the black vs white groups caused by random sign flips.

\begin{figure}[ht]
\begin{center}
   \includegraphics[width=0.70\linewidth]{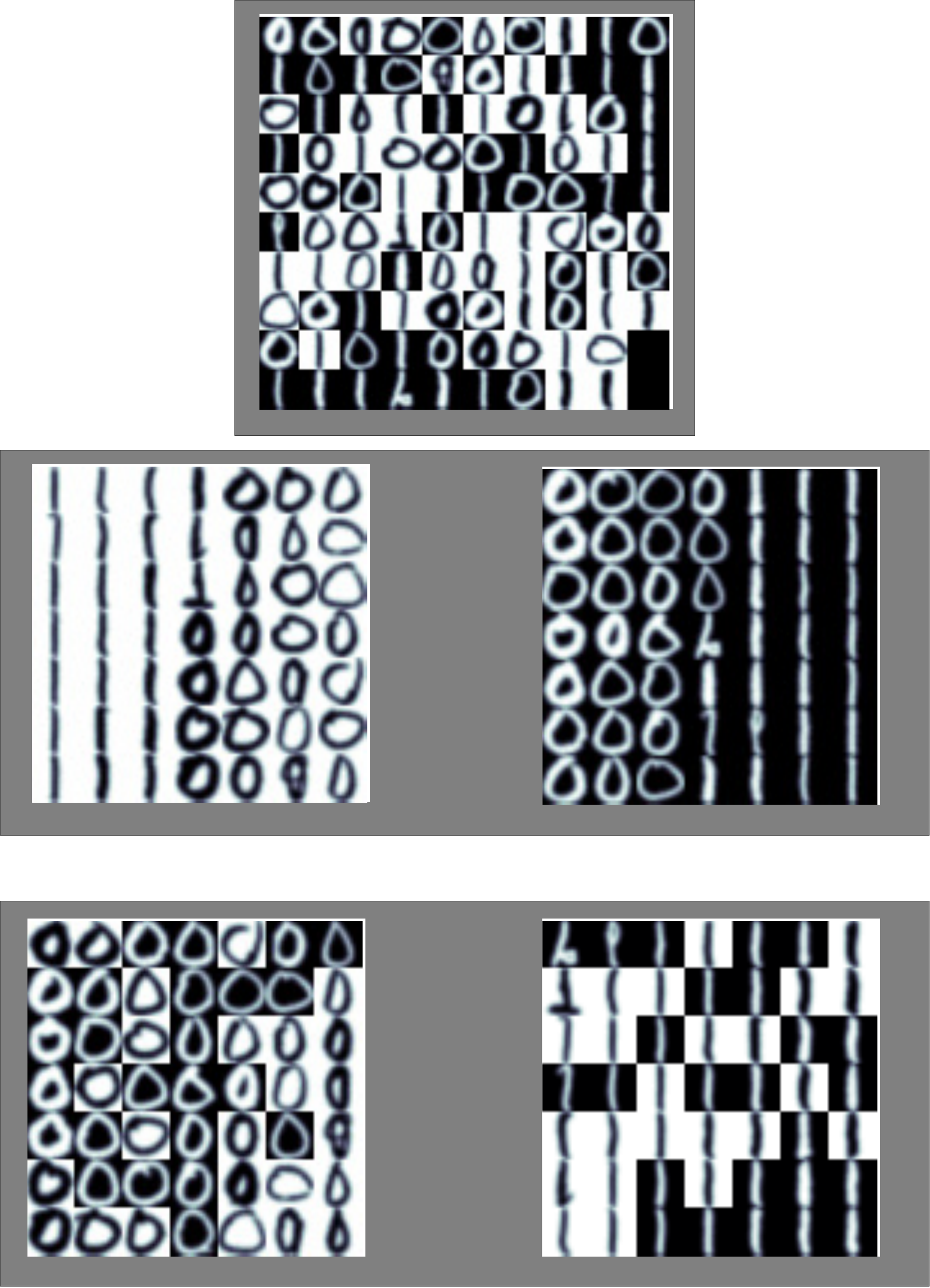}
\end{center}
\caption{(Top) 98 Handwritten digits from the USPS dataset with randomly switched signs. (Middle) Spectral Clustering using a Gaussian Kernel $(\sigma=22)$ and two clusters. (Bottom)  Spectral Clustering using the sign-invariant Gaussian Kernel $(\sigma=22)$  and two clusters.
  }\label{fig:digits}
\end{figure}

\begin{figure}[h]
\begin{center}
  \includegraphics[width=0.3\linewidth]{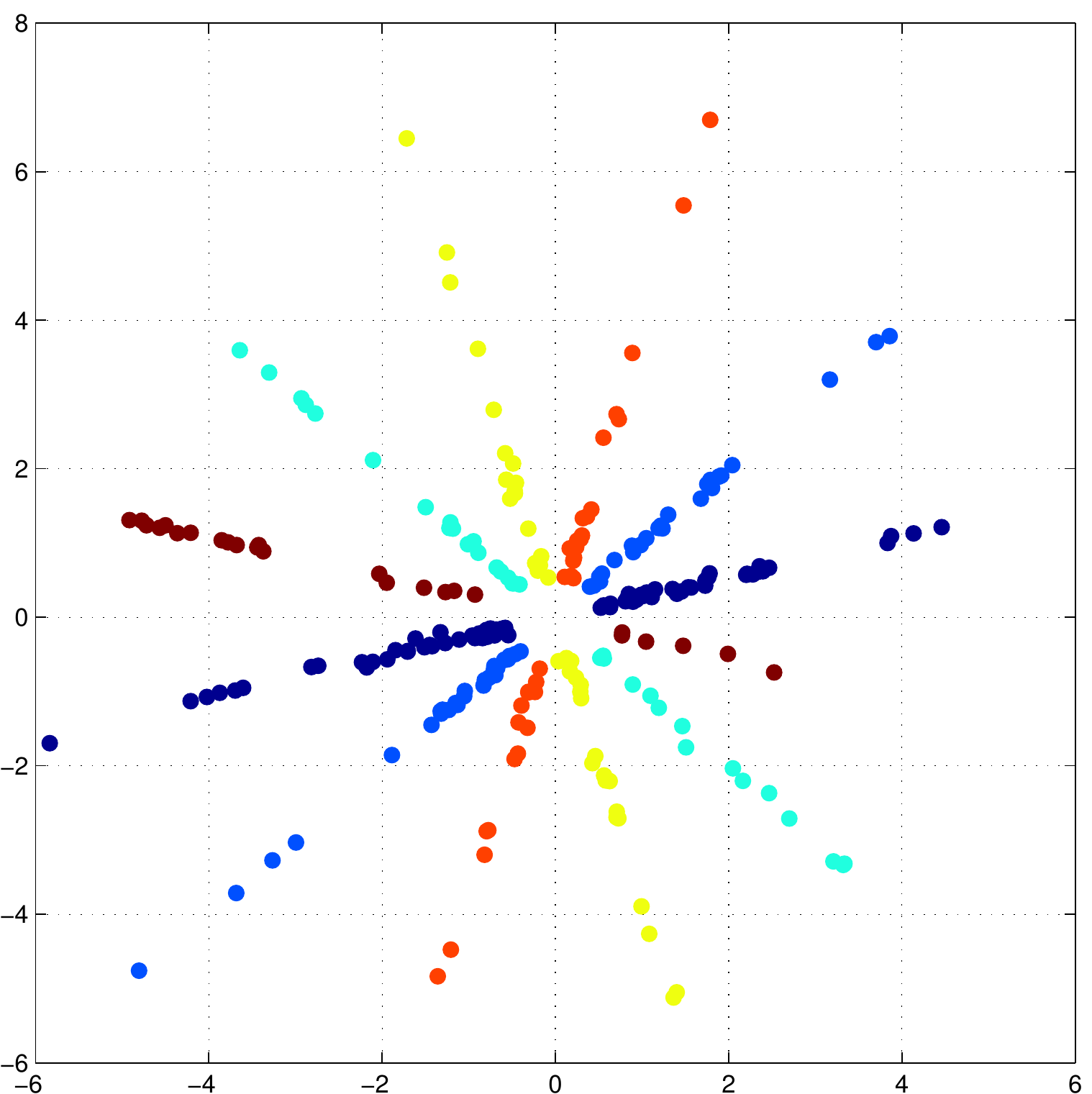}
 \includegraphics[width=0.3\linewidth]{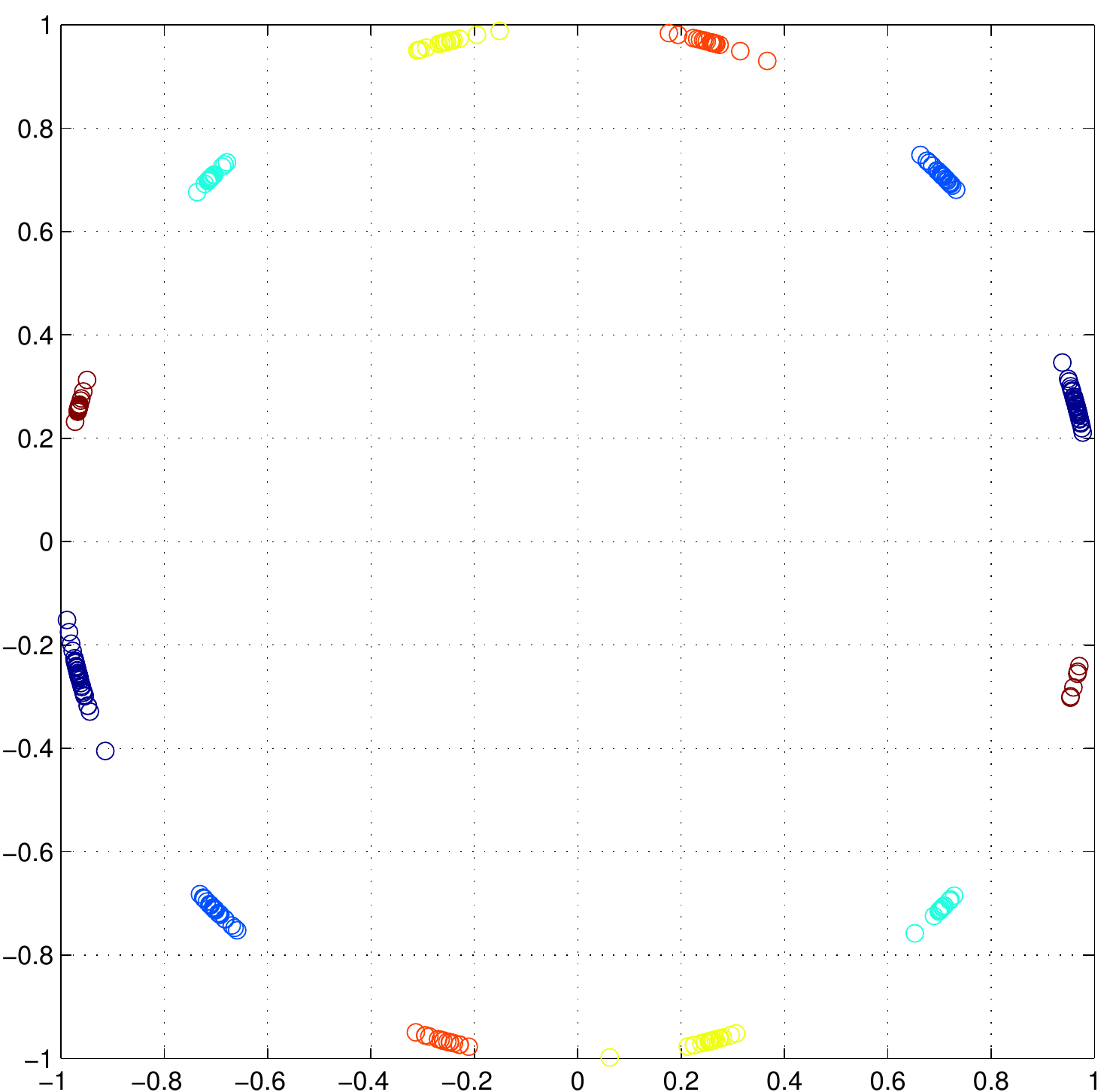}
 \includegraphics[width=0.6\linewidth]{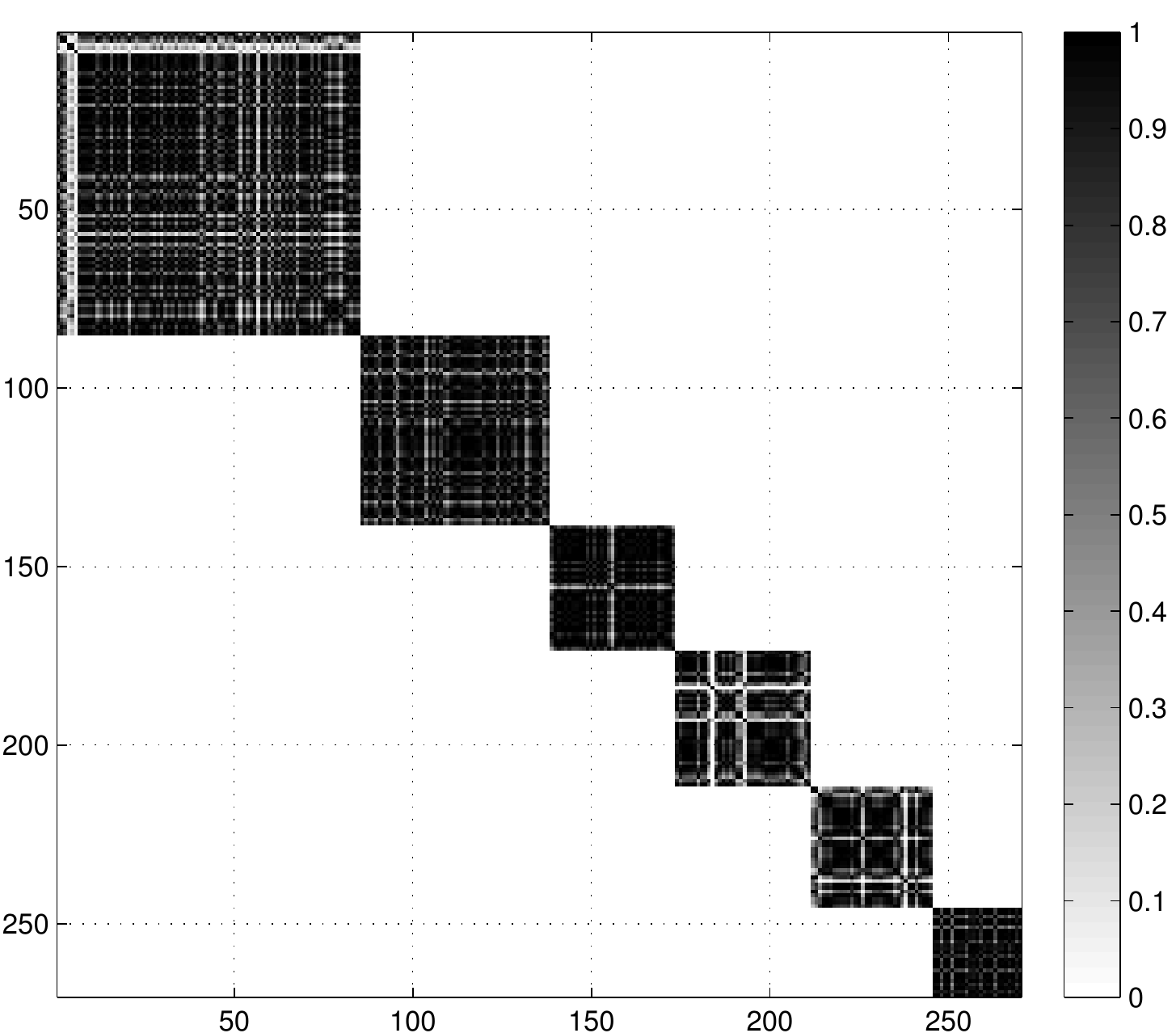}
\end{center}
\caption{
 Here we applied spectral clustering using our sign-invariant Gaussian kernel  (\ref{eq:pgk}) with $\sigma=0.1$ and 6 clusters.
(Top left panel) Clustering of sparse wavelet coefficients to solve an overcomplete ICA problem. There are six directions corresponding to six sources measured with two sensors. (Top right panel) Normalized data to obtain scaling invariance.
(Bottom panel) Rearranged kernel matrix shows a clear block structure. }
\label{fig:flutes}
\end{figure}

\subsection{Clustering for overcomplete ICA }

The main motivation for clustering with sign invariance comes from matrix
factorization methods such as ICA \citep{HyKaOj01} where the solutions are only
determined up to scaling (which can easily be addressed by
normalization) and sign or phase factors (which are more difficult to
address otherwise).
We note that the task of clustering of ICA components is important in group studies of EEG or MEG \citep{Spadone2012}.

Beyond the task of grouping ICA components, it is also possible to perform ICA  by clustering of sparse representations of the observed mixtures obtained by e.g. short-time Fourier- or wavelet-decompositions. In this
scenario the clustering approach enables the identification and
recovery of more sources then sensors, i.e. it solves an overcomplete
ICA problem
\citep{CheDonSau98,BofZib01,ZibPea01,HyKaOj01,MeiHarMue04a}.


In this example we use the ``six-flutes'' dataset of \citet{BofZib01}
which consists of two instantaneous linear mixtures $x(t)=A s(t)$ of
recordings of six different notes played with a flute. The mixing matrix $A$ has
dimensions $(2 \times 6)$.

 Applying a cosine packet tree \citep{CheDonSau98} for sparsification
and selecting 270 points with largest norm yields the dataset shown in
the left panel of figure \ref{fig:flutes}. The columns of the mixing
matrix $A$ correspond to 6 equally spaced directions in the two
dimensional space of the mixed signals, i.e.~there are six directions
representing six sources measured with two sensors.
Clustering this data is equivalent to a  ``blind'' identification of the mixing matrix $A$ up to scale and sign indeterminacies.
 In the bottom panel of figure  \ref{fig:flutes} we show the entries of the kernel matrix computed with the sign-and-scale invariant Gaussian kernel. There is a clear block structure with six blocks corresponding to six sign-and-scale invariant clusters.

\section{Conclusion}\label{sec:conclusion}

Algorithms for invariant pattern recognition have received continuous
attention over the last decades.

In this paper we propose a theoretical framework for kernel methods
that allows the systematic and constructive incorporation of algebraic
invariance structure into a Mercer Kernel.
In this manner known structural invariances can be
implemented by applying the kernel trick twice: first a nonlinear
inner kernel is constructed that hard codes the invariance and then a
second nonlinear kernel is applied to the result. The additional
computational load involved is negligible, but the gain in performance
and meaningfulness of the learned result is substantial. While the
proposed framework is general -- we showed how to code a variety of
invariances -- also its potential practical applications are numerous;
in fact, any kernel algorithm can be made invariant in this manner.

We have limited ourselves to the application of spectral clustering and
demonstrated exemplarily how to incorporate sign invariance and sign-and-scaling invariance; our experiments have validated the usefulness of our approach.

We remark that being able to construct sign-invariant, or
sign-and-scale invariant clustering algorithms is actually a pertinent
issue in biomedical applications where decomposition methods such as
PCA, ICA, CSP etc.~yield components that are to be grouped with low
computational cost, since clearly, combinatorial heuristic solutions
are infeasible in the context of the now commonly available
multi-channel imaging systems.

Here our novel algorithms will already be of practical use. Note,
however, that the spirit of this contribution is mainly considered foundational and conceptual.

Future work will be devoted to incorporate even more general
invariance structure and to the broader application of the proposed
framework in the fields of Bioinformatics and Biomedical Engineering.

\section*{Acknowledgements}
KRM and AZ gratefully acknowledge funding by DFG and BMBF. KRM thanks for partial funding by the National Research
Foundation of Korea funded by the Ministry of Education, Science, and Technology in the BK21 program. This research was carried out at MFO, supported by FK's Oberwolfach Leibniz Fellowship.

\section*{Appendix}

In the following, we provide proofs for several of the statements in the corpus regarding invariances.

\subsubsection*{Quotient rings and quotient varieties}

In the following, commutative rings will always contain the one-element.

\begin{Def}
Let $R$ be a commutative ring, and $G$ a group acting on $R$. Then, we define
$$R^G:=\{r\in R\;:\; g.r = r\;\mbox{for all}\; g\in G\}.$$
and call $R^G$ the \emph{invariant ring} of $R$ w.r.t~$G$.
\end{Def}

A classical characterization of the invariant ring is by its universal property:

\begin{Lem}\label{Lem:grpinv}
Let $\phi: S\rightarrow R$ be a homomorphism of rings, and $G$ a group acting on $R$. Let $\iota: R^G\rightarrow R$ be the canonical embedding. If $g\circ \phi = \phi$ for all $g\in G$, then there exists a unique homomorphism $\varphi: S\rightarrow R^G$, such that $\iota\circ\varphi = \phi$.
\end{Lem}
\begin{proof}
If suffices to show that the image of $\phi$ is contained in $R^G$. But $g\circ \phi = \phi$ for all $g\in G$ implies $g.\phi(s) = \phi(s)$ for all $g\in R, s\in S$, which implies that $\phi(s)\in R^G$ for all $s\in S$, which proves the claim.
\end{proof}

\begin{Def}
Let $V,W$ be algebraic varieties, let $G$ be a group acting on $V$. A homomorphism $f:V\rightarrow W$ is called $G$-invariant if $f\circ g = f$ for all $g\in G$.
\end{Def}

We proceed with the geometric analogue of Lemma~\ref{Lem:grpinv}, which is the analogue universal property for the scheme-theoretic quotient, which we formulate in a form less general than usual in algebraic geometry, but which fits the specific setting outlined in the main corpus.

\begin{Prop}\label{Prop:invs}
Let $V\subseteq \KK^n$ be an affine algebraic variety, let $G$ be a group algebraically acting on $V$. Then, there is a family of algebraic invariant maps
$$q: V\rightarrow V/G \subseteq \KK^I \quad\mbox{, i.e.,} \quad \left(q_i:\KK^n\rightarrow \KK\right)_{i\in I},$$
such that for all varieties $W\subseteq \KK^m $ and $G$-invariant homomorphisms $f:V\rightarrow W$ there is a unique homomorphism $g:V/G\rightarrow W$ such that $f = g\circ q$.
Furthermore, $I$ is countable, and if $G$ is finite, then so is $I$.
\end{Prop}
\begin{proof}
By the algebra-geometry duality (contravariant equivalence of commutative rings and affine schemes), the homomorphism of varieties $f: V\rightarrow W$ is dual to a homomorphism of rings $\phi: S\rightarrow R$, with $W = \Spec S$ and $V = \Spec R$, and $G$ canonically acting on $R$. By Lemma~\ref{Lem:grpinv}, $\phi$ factors uniquely as $\iota\circ\varphi = \phi$, into $\varphi: S\rightarrow R^G$, and the canoncial quotient morphism $\iota: R^G\rightarrow R$. Taking $V/G := \Spec R^G$ and $q:=\Spec(\iota), g := \Spec (q)$ yields the first claim on existence of $q$ and factorization of $f$.
For the second statement, note that $V$ is contained in $\KK^n$, thus $R$ is finitely generated over $\KK$, thus $R$ is an $\KK$-vector space of countable dimension. As a sub-vector space, $R^G$ is also of countable dimension over $\KK$. As $R^G$ is generated by any vector space basis, it is countably generated over $\KK$, this $I$ can be taken countable. Moreover, if $G$ is finite, then by Hilbert's invariant theorem, $R^G$ is finitely generated over $\KK$ as well.
\end{proof}

We would like to note that if $G$ is infinite, it can both happen that $V/G$ is of finite type or not. For example, an infinite but linearly reductive $G$ gives rise to a finite type quotient, as well as any $G$ in the case $n=1$ or $2$, see~\citep{Zariski54}. But $V/G$ can also be not of finite type, as any counterexample to Hilbert's 14-th problem shows, such as the example of~\cite{Totaro2008} in the case $n=16$ which can be invoked with a real or complex ground field $\KK$.

\subsubsection*{The universal property of the quotient}

We proceed with the proof of Theorem~\ref{Thm:quot} which is essentially a minor variant of Proposition~\ref{Prop:invs} above and translates the universal property of the quotient to a setting more suitable for kernels.

{\bf Proof of Theorem~\ref{Thm:quot}}.
Consider first the case $\calF = \KK$.

Note that $G$ acts canonically on the $\KK$-vector space $C(W, \calF)$, with an vector space of invariant functions $C(W, \calF)^G$.
Now, by the Stone-Weierstra\ss-Theorem, the algebraic/polynomial functions $\Pol (W,\calF)$ are dense in the continuous functions $C(W, \calF)$. Taking the intersection with $C(W, \calF)^G$, it follows that the $G$-invariant algebraic/polynomial functions $\Pol (W,\calF)^G$ are dense in $C(W, \calF)^G$.

Now (by categorical equivalence of evaluation homomorphisms and polynomials over the characteristic zero field $\KK$), the space $\Pol (W,\calF)$ is canonically isomorphic to the ring $R=\KK [X_1,\dots, X_n]/\Id (W)$, while $\Pol (W,\calF)^G$ is isomorphic to $R^G$. Note that $\Id(W) = \Id(Z)$, where $Z$ is the Zariski closure of $W$ in $\KK^n$. Thus, $\Pol (W,\calF)$ is isomorphic to $\Pol (Z,\calF)$, and $\Pol (W,\calF)^G$ is isomorphic to $\Pol (Z,\calF)^G$, showing that in the above argumentation $W$ can be replaced by the variety $Z$.

Proposition~\ref{Prop:invs} yields, for every $f\in \Pol (Z,\calF)^G$, a decomposition of the desired kind, thus by isomorphism for any $f\in \Pol (W, \calF)^G$, thus by denseness for any $f\in C(W, \calF)^G$.

For general $\calF$, one obtains the analogue statement by passing to a basis representation $\calF = \spn (f_i, i\in I)$, yielding a decomposition $\Pol (W,\calF) = \bigoplus_{i\in I} \Pol (W,\KK).$

Note that the Hilbert space $\calI$ in the statement is therefore isomorphic to the dual of $\Pol (W,\calF)^G$.

\subsubsection*{Isometric invariants}

\begin{Lem}\label{Lem:rot}
Consider the orthogonal/unitary group $G=\Orth(n)$ resp.~$G= \Unit(n)$ (consisting of orthogonal/unitary matrices) which acts on $V=\KK^n$, i.e, $A.x = A\cdot x.$\\
Then the invariant map is $q:V\rightarrow V/G\subseteq \RR$ given by
$x \mapsto \langle x,x\rangle.$
\end{Lem}
\begin{proof}
Observe that any $x\in \KK^n$ is $G$-equivalent to $\|x\|\cdot e_1$, where $e_1$ is a standard normal basis vector, and the orbits of $\|x\|\cdot e_1$ are distinct. Alternatively, instantiate the diagonal invariant map from Lemma~\ref{Lem:diagonalrot} on vectors $(x,0)$.
\end{proof}

\begin{Lem}\label{Lem:diagonalrot}
Consider the orthogonal/unitary group $G=\Orth(n)$ resp.~$G= \Unit(n)$ (consisting of orthogonal/unitary matrices) which acts diagonally on $V=\KK^n\times \KK^n$, i.e, $A.(x,y) = (A\cdot x, A\cdot y).$\\
Then the invariant map is $q:V\rightarrow V/G\subseteq \RR\times \KK\times \RR$ given by
$$(x,y)\mapsto \left(\langle x,x\rangle,\langle x,y\rangle, \langle y,y\rangle \right).$$
\end{Lem}
\begin{proof}
This is shown in Lemma 6.1.1~of~\citep{Cro09}.
\end{proof}

\begin{Lem}\label{Lem:diagonaltrans}
Consider the translational group $G=\Trans(n)\cong \KK^n$ which acts diagonally on $V=\KK^n\times \KK^n$, i.e, $v.(x,y) = (x + v, y + v).$\\
Then the invariant map is $q:V\rightarrow V/G\subseteq \KK^n$ given by
$(x,y)\mapsto \left( x - y \right).$
\end{Lem}
\begin{proof}
The action is linear, thus the group quotient as variety equals the vector space quotient, which is the image of the linear map $q$.
\end{proof}

\begin{Lem}\label{Lem:diagonalrottrans}
Consider the Euclidean group $G=\Euc (n)$ generated by translations and orthogonal/unitary rotations which acts diagonally on $V=\KK^n\times \KK^n$.\\
Then the invariant map is $q:V\rightarrow V/G\subseteq \RR$ given by
$(x,y)\mapsto \left\| x - y \right\|^2 = \langle x-y,x-y\rangle.$
\end{Lem}
\begin{proof}
Recall that $G$ is generated by the translations $\Trans(n)$ and orthogonal/unitary matrices $H=\Orth(n)$ resp.~$H=\Unit(n)$, that $\Trans(n)$ is a normal subgroup of $G=\Euc (n)$, and the corresponding factor group is $\Euc(n) / \Trans(n) = H$. Thus, $V/G = (V/\Trans(n))/H$.
By Lemma~\ref{Lem:diagonaltrans}, there is a canonical quotient map
$$q_T: V\rightarrow V/\Trans(n) = \KK^n, (x,y)\mapsto \left( x - y\right),$$
thus $(V/\Trans(n))/H = V/G = \KK^n/H$. By Lemma~\ref{Lem:rot}, there is a further canonical quotient map
$$q_H: V/\Trans(n)\rightarrow V/G = \RR, x\mapsto \langle x,x\rangle = \|x\|^2,$$
which after composition $q:=q_H\circ q_T$ yields the claim.
An alternative proof can be obtained in analogy to that of Lemma~\ref{Lem:rot}, by observing and verifying in an elementary way that any $(x,y)\in \KK^n\times \KK^n$ is diagonally $G$-equivalent to $\|x-y\|\cdot e_1$, and two such vectors are distinct if and only if their absolute value is.
\end{proof}

\bibliographystyle{plainnat}
\bibliography{invariance}

\end{document}